\documentclass[]{article}
\usepackage{aliascnt}
\usepackage{fullpage}

\usepackage{graphicx}
\usepackage{balance}  
\usepackage{subfig}  
\usepackage{xspace}
\usepackage[usenames,dvipsnames]{color}
\usepackage{epsfig} 
\usepackage{algorithm}
\usepackage{algorithmic} 
\usepackage{hyperref}
\usepackage{amsmath, amssymb, amsthm}

\newcommand{\cl}[1]{\ensuremath {\sf #1}}
\newcommand{\set}[1]{\{#1\}}  
\newcommand{\sharpP}{\cl{\# P}}

\newtheorem{theorem}{Theorem}[section]          	
\newaliascnt{lemma}{theorem}				
\newtheorem{lemma}[lemma]{Lemma}              	
\aliascntresetthe{lemma}  					
\newaliascnt{conjecture}{theorem}			
\aliascntresetthe{conjecture}  				
\newaliascnt{remark}{theorem}				
              
\aliascntresetthe{remark}  					
\newaliascnt{corollary}{theorem}			
\newtheorem{corollary}[corollary]{Corollary}      
\aliascntresetthe{corollary}  				
\newaliascnt{definition}{theorem}			
\newtheorem{definition}[definition]{Definition}    
\aliascntresetthe{definition}  				
\newaliascnt{proposition}{theorem}			
\newtheorem{proposition}[proposition]{Proposition}  
\aliascntresetthe{proposition}  				
\newaliascnt{example}{theorem}			
\aliascntresetthe{example}  				

\newcommand{\cut}[1]{\relax}

\usepackage{relsize}
\usepackage{latexsym}
\usepackage{framed}

\newcommand{\SDD}{\alpha}
\newcommand{\vtree}{\bf{v}}

\newcommand{\desc}{\operatorname{Vars}}
\newcommand{\shell}{\operatorname{Shell}}
\newcommand{\parent}{\operatorname{Parent}}

\newcommand{\row}{\operatorname{Row}}
\newcommand{\col}{\operatorname{Col}}

\newcommand{\lab}{\operatorname{Sdds}}

\newcommand{\stack}{\operatorname{Chain}}

\newcommand{\CC}{\operatorname{CC}}
\newcommand{\DNNF}{\mathcal{D}}

\title{New Limits for Knowledge Compilation and Applications to Exact Model Counting} 

\author{ {\bf Paul Beame\thanks{\ \ Research supported by NSF grant CCF-1217099.}} \\  
Computer Science and Engineering\\
University of Washington\\
Seattle, WA 98195\\
{\tt beame@cs.washington.edu}
\and 
{\bf Vincent Liew}$^*$\\ 
Computer Science and Engineering\\
University of Washington\\
Seattle, WA 98195\\
{\tt vliew@cs.washington.edu}
} 
 
\begin{document} 
 
\maketitle 
 

\begin{abstract}
We show new limits on the efficiency of using current
techniques to make exact probabilistic inference
for large classes of natural problems.
In particular we show new lower bounds on knowledge compilation to SDD
and DNNF forms.  
We give strong lower bounds on the complexity of SDD
representations by relating SDD size to best-partition 
communication complexity.
We use this relationship to prove exponential lower bounds on the SDD size
for representing a large class of problems that occur naturally as
queries over probabilistic databases.
A consequence is that for representing unions of conjunctive queries,
SDDs are not qualitatively more concise than OBDDs.
We also derive simple examples for which SDDs must be exponentially less
concise than FBDDs.
Finally, we derive exponential lower bounds on the sizes of DNNF
representations using a new quasipolynomial simulation of DNNFs by
nondeterministic FBDDs.
\end{abstract}

\section{Introduction}
Weighted model counting is a fundamental problem in probabilistic inference
that captures the computation of probabilities of complex predicates over
independent random events (Boolean variables).
Although the problem is \sharpP-hard in general, 
there are a number of practical algorithms for model counting based on
DPLL algorithms and on knowledge compilation techniques.
The knowledge
compilation approach, though more space intensive, can be much more convenient
since it builds a representation for an input predicate independent of its
weights that allows the count to evaluated easily given a particular choice
of weights; that representation also can be re-used to analyze more
complicated predicates.
Moreover, with only a constant-factor increase in time, the 
methods using DPLL algorithms can be easily extended to be knowledge compilation
algorithms~\cite{HuangDJAIR07}.
(See~\cite{GomesSS09} for a survey.)

The representation to be used for knowledge compilation is an important key to
the utility of these methods in practice; the best methods are based on
restricted classes of circuits and on decision diagrams.  All of the ones
considered to date can be seen as natural sub-classes of the class of
{\em Decomposable Negation Normal Form (DNNF)} formulas/circuits introduced
in~\cite{Darwiche01JACM}, though it is not known how to do model counting
efficiently
for the full class of DNNF formulas/circuits.  One sub-class for which model
counting is efficient given the representation is that of {\em d-DNNF}
formulas, though there is no efficient algorithm known to recognize whether a
DNNF formula is d-DNNF.   

A special case of d-DNNF formulas (with a minor change of syntax)
that is easy to recognize is that of {\em decision-DNNF} formulas.  
This class of representations captures all of the
practical model counting algorithms discussed in~\cite{GomesSS09} including
those based on DPLL algorithms.
Decision-DNNFs include
{\em Ordered Binary Decision Diagrams (OBDDs)}, which are canonical and
have been highly effective representations for verification,
and also {\em Free} BDDs (FBDDs),
which are also known as read-once branching programs.  
Using a quasi-polynomial simulation of
decision-DNNFs by FBDDs,
\cite{Beame+13-uai,Beame+14-icdt} showed
that the best decision-DNNF representations must be
exponential even for many very simple 2-DNF predicates that
arise in probabilistic databases.

Recently, \cite{DBLP:conf/ijcai/Darwiche11} introduced another
subclass of
d-DNNF formulas called {\em Sentential Decision Diagrams (SDDs)}.  This
class is strictly more general than OBDDs and  (in its basic form) is
similarly canonical.  (OBDDs use a fixed ordering of variables, while
SDDs use a fixed binary tree of variables, known as a {\em vtree}.)
There has been
substantial development and growing application of SDDs to knowledge
representation problems, including a recently released SDD software
package~\cite{SDDpackage}.   Indeed, SDDs hold potential to be more concise
than OBDDs.
\cite{DBLP:conf/aaai/BroeckD15} showed that {\em compressing} an SDD
with a fixed vtree so that it is canonical can lead to an exponential blow-up
in size, but much regarding the complexity of SDD representations has remained
open.

In this paper we show the limitations both of general DNNFs and especially
of SDDs.  We show that
the simulation of decision-DNNFs by FBDDs from~\cite{Beame+13-uai}
can be extended to yield a simulation of general DNNFs by OR-FBDDs, the
nondeterministic extension of FBDDs, from which we can derive exponential
lower bounds for DNNF
representations of some simple functions. 
This latter simulation, as well as that of~\cite{Beame+13-uai},
is tight, since~\cite{DBLP:journals/corr/Razgon15a} (see also \cite{DBLP:conf/kr/Razgon14}) shows a quasipolynomial separation between DNNF and OR-FBDD size
using parameterized complexity. 

For SDDs we obtain much stronger results.  
In particular, we relate the SDD size required to represent predicate $f$
to the "best-case partition" communication
complexity~\cite{kn97} of $f$.  Using this, together with reductions to the
communication complexity of disjointness (set intersection), we derive the
following results:\\
(1) There are simple predicates given by 2-DNF formulas for which FBDD size is polynomial but for 
which SDD size must be exponential.\\
(2) For a natural, widely-studied class of database queries 
known as {\em Unions of Conjunctive Queries (UCQ)}, the SDD size is linear
iff the OBDD size is linear and is exponential otherwise (which corresponds
to a query that contains an {\em inversion} \cite{DBLP:journals/mst/JhaS13}).\\
(3) Similar lower bounds apply to the dual of UCQ, which consists of universal, positive queries.

To prove our SDD results, we show that for any predicate $f$ given by
an SDD of size $S$, using its associated vtree
we can partition the variables of $f$ between two players, Alice and Bob,
in a nearly balanced way so that they only need to send $\log^2 S$ bits of
communication to compute $f$.  The characterization goes through an 
intermediate step involving {\em unambiguous} communication protocols
and a clever deterministic simulation of such protocols from~\cite{yannakakis:extension}.

\begin{sloppypar}
\paragraph{Related work:}  
The quasi-polynomial simulation of DNNFs by OR-FBDDs that we give was also shown
independently in~\cite{DBLP:conf/cp/Razgon15}. 
Beyond the lower bounds for
decision-DNNFs in \cite{Beame+13-uai,Beame+14-icdt} which
give related analyses for decision-DNNFs, the
work of \cite{DBLP:conf/aaai/PipatsrisawatD10} on structured DNNFs
is particularly relevant to this paper\footnote{We thank the conference reviewers
for bringing this work to our attention.}.
\cite{DBLP:conf/aaai/PipatsrisawatD10} show how sizes of
what they term {\em (deterministic) $\bf X$-decompositions} can yield lower
bounds on the sizes of {\em structured (deterministic)} DNNFs, which include
SDDs as a special case. 
\cite{pipatsrisawat:dissertation} contains the full details of how this can be
applied to prove lower bounds for specific predicates.  These bounds are
actually equivalent to lower bounds exponential in the best-partition
nondeterministic (respectively, unambiguous) communication complexity of the
given predicates.  Our paper derives this lower bound for SDDs directly but,
more importantly, provides the connection to best-partition {\em deterministic}
communication complexity,
which allows us to have a much wider range of application; this strengthening
is necessary for our applications.
Finally, we note that \cite{DBLP:conf/kr/Razgon14} showed that SDDs can be
powerful by finding examples where OBDDs using any order are quasipolynomially
less concise than SDDs.
\end{sloppypar}

\paragraph{Roadmap:}
We give the background and some formal definitions including some generalization
required for this work in Section~\ref{background}. 
We prove our characterization of SDDs in terms of best-partition communication complexity in Section~\ref{sdd-lower} and derive the resulting bounds for SDDs
for natural predicates in Section~\ref{applications}.
We describe the simulation of DNNFs by OR-FBDDs, and its consequences, in
Section~\ref{dnnf-simulation}.

\section{Background and Definitions}
\label{background}

We first give some basic definitions of DNNFs and decision diagrams.

\begin{definition}
A \emph{Negation Normal Form (NNF)} circuit is a Boolean circuit 
with $\lnot$ gates, which may only be applied to inputs, and $\lor$
and $\land$ gates.  Further, it is {\em Decomposable (DNNF)} iff the
children of each $\land$ gate are reachable from disjoint sets of input
variables. 
(Following convention, we call this circuit a ``DNNF formula'',
though it is not a Boolean formula in the usual sense of circuit complexity.)
A DNNF formula is \emph{deterministic (d-DNNF)} iff the functions computed at
the children of each $\lor$ gate are not simultaneously satisfiable.
\end{definition}

\begin{definition}
A \emph{Free Binary Decision Diagram (FBDD)} is a directed acyclic graph 
with a single source (the root)
and two specified sink nodes, one labeled 0 and the other 1.  Every non-sink
node is labeled by a Boolean variable and has two out-edges, one labeled 0
and the other 1.  No path from the root to either sink is labeled by the 
same variable more than once.  It is an \emph{OBDD} if the order of variable
labels is the same on every path.  The Boolean function computed by an FBDD is
1 on input $\bf a$ iff there is a path from the root to the sink labeled 1 so
that for every node label $X_i$ on the path, ${\bf a}_i$ is the label of the
out-edge taken by the path.
An \emph{OR-FBDD} is an FBDD augmented with additional nodes of
arbitrary fan-out labeled $\lor$.  The function value for the OR-FBDD follows
the same definition as for FBDDs; the $\lor$-nodes simply make more than one
path possible for a given input. (See~\cite{wegener2000book}.)
\end{definition}

We now define sentential decision diagrams as well as a small generalization
that we will find useful.

\begin{definition}
For a set ${\bf X}$, let $\top: \{0,1\}^{\bf X} \rightarrow \{0,1\}$ and
$\bot:\{0,1\}^{\bf X} \rightarrow \{0,1\}$  denote the constant $1$ function
and constant $0$ function, respectively.
\end{definition}

\begin{definition}
We say that a set of Boolean functions $\{p_1, p_2, \ldots,p_\ell\}$, where each
$p_i$ has domain $\set{0,1}^{\bf X}$, is \emph{disjoint} if for each
$i \not= j$, $p_i \wedge p_j = \bot$. We call $\{p_1, p_2, \ldots,p_\ell\}$
a \emph{partition} if it is disjoint and $\bigvee_{i=1}^{\ell} p_i = \top$.
\end{definition}



\begin{definition}
A \emph{vtree} for variables $\bf{X}$ is a full binary tree whose leaves are
in one-to-one correspondence with the variables in $\bf{X}$.
\end{definition}

We define {\em Sentential Decision Diagrams (SDDs)} together with the
Boolean functions they represent and use
$\langle . \rangle$ to denote the mapping from SDDs into Boolean
functions.  (This notation is extended to sets of SDDs yielding sets of 
Boolean functions.)
At the same time, we also define a directed acyclic graph (DAG) representation
of the SDD.

\begin{definition}
$\SDD$ is an {\em SDD that respects vtree} $\vtree$ rooted at $v$ iff:
\begin{itemize}
\item $\SDD = \top$ or $\SDD = \bot$.   
\\ Semantics: $\langle \top \rangle = \top$ and $\langle \bot \rangle = \bot$.
\\$G(\SDD)$ consists of a single leaf node labeled with $\langle \SDD\rangle$.

\item $\SDD = X$ or $\SDD = \neg X$ and $v$ is a leaf with variable $X$.
\\ Semantics: $\langle X \rangle = X$ and $\langle \neg X \rangle = \neg X$
\\$G(\SDD)$ consists of a single leaf node labeled with $\langle \SDD\rangle$.

\item $\SDD = \{(p_1,s_1),\ldots,(p_\ell,s_\ell)\}$, $v$ is an internal vertex
with children $v_L$ and $v_R$, $p_1,\ldots,p_\ell$ are SDDs that respect the
subtree rooted at $v_L$,
$s_1,\ldots,s_\ell$ are SDDs that respect the subtree rooted at $v_R$, and
$\langle p_1 \rangle, \ldots, \langle p_\ell \rangle$ is a partition.
\\ Semantics: $\langle \SDD \rangle = \bigvee_{i=1}^n \big(\langle p_i \rangle \wedge \langle s_i \rangle\big)$
\\$G(\SDD)$ has a circle node for $\SDD$ labeled $v$ with $\ell$ child box
nodes labeled by the pairs $(p_i,s_i)$.  
A box node labeled $(p_i,s_i)$ has a left child that is 
the root of $G(p_i)$ and and a right child that is the root of $G(s_i)$. The
rest of $G(\SDD)$ is the (non-disjoint) union of graphs
$G(p_1),\ldots,G(p_\ell)$ and $G(s_1),\ldots,G(s_\ell)$ with common sub-DAGs
merged.  (See Figure~\ref{fig:SDD}.)
\end{itemize}
Each circle node $\SDD'$ in $G(\SDD)$ itself represents an SDD that respects
a subtree of $\vtree$ rooted at some vertex $v'$ of $\vtree$; 
We say that $\SDD'$ is {\em in} $\SDD$ and use $\lab(v',\SDD)$ to denote the
collection of $\SDD'$ in $\SDD$ that respect the subtree rooted at $v'$.
The \emph{size} of an SDD $\SDD$ is the number of nodes in
$G(\SDD)$.
\end{definition}

Circle nodes in $G(\SDD)$ may be interpreted as OR gates and paired box nodes
may be interpreted as AND gates. In the rest of this paper, we will view
SDDs as a class of Boolean circuit.
The vtree property and partition property of SDDs together ensure that this
resulting circuit is a d-DNNF.

\begin{figure}[t]
  \centering 
\includegraphics[scale = 0.6]{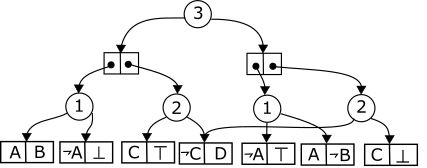}\hfil \includegraphics[scale=0.30]{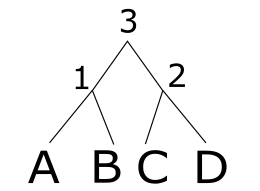}
\caption{An SDD with its associated vtree that computes the formula 
$(A \wedge B \wedge C) \vee (\neg C \wedge D)$}
\label{fig:SDD}
\end{figure}

We define a small generalization of vtrees which will be useful for describing
SDDs with respect to a partial assignment of variables.

\begin{definition}
A \emph{pruned vtree} for variables $\bf{X}$ is a full binary tree whose
leaves are either marked \emph{stub} or by a variable in $\bf{X}$, and
whose leaves marked by variables are in one-to-one correspondence with the
variables in $\bf{X}$.  
\end{definition}

We generalize SDDs so that they can respect pruned vtrees.
The definition is almost identical to that for regular SDDs so we
only point out the differences.

\begin{definition}
The definition of a {\em pruned SDD $\SDD$ respecting a pruned vtree} $\vtree$,
its semantics, and its graph $G(\alpha)$, are
identical to those of an SDD except that \\[-3ex]
\begin{itemize}
\item if the root vertex $v$ of $\vtree$ is a stub then $\langle \SDD \rangle$
 must be $\bot$
or $\top$, and
\item if the root vertex $v$ of $\vtree$ is internal then 
we only require that $\langle p_1 \rangle, \ldots, \langle p_\ell \rangle$ are
disjoint but not necessarily that they form a partition.
\end{itemize}
%
%
%
\end{definition}

We now sketch a very brief overview of the communication complexity we
will need.
Many more details may be found in~\cite{kn97}.
Given a Boolean function $f$ on $\set{0,1}^{\bf X}\times\set{0,1}^{\bf Y}$,
one can define two-party protocols in which two players, Alice, who receives 
$x\in \set{0,1}^{\bf X}$ and Bob, who receives $y\in \set{0,1}^{\bf Y}$ 
exchange a sequence of messages $m_1,\ldots, m_C=f(x,y)\in \set{0,1}$ to
compute $f$.
(After each bit, the player to send the next bit must be determined from
previous messages.)
The {\em (deterministic) communication complexity} of $f$,
$CC(f({\bf X},{\bf Y}))$,
is the minimum value $C$ over all protocols computing $f$
such that all message sequences are of length at most $C$.
The \emph{one-way deterministic communication complexity} of $f$,
$CC_{{\bf X} \rightarrow {\bf Y}}(f({\bf X},{\bf Y}))$ is the
minimum value of $C$ over all protocols where Alice may send messages
to Bob, but Bob cannot send messages to Alice.

For nondeterministic protocols, Alice simply guesses a string based on her
input $x$ and sends the resulting message $m$ to Bob, who uses $m$ together with
$y$ to verify whether or not $f(x,y)=1$.
The communication complexity in this case is the minimum $|m|$ over all
protocols.  Such a protocol is {\em unambiguous} iff for each $(x,y)$ pair
such that $f(x,y)=1$ there is precisely one message $m$ that will cause Bob
to output 1.  A set of the form $A\times B$ for $A\subseteq \set{0,1}^{\bf X}$,
$B\subseteq \set{0,1}^{\bf Y}$ is called a \emph{rectangle}.  The minimum of
$|m|$ over all unambiguous protocols is the
{\em unambiguous communication complexity} of $f$;
it is known to be the logarithm base 2 of the minimum number of rectangles
into which one can partition the set of inputs on which $f$ is 1.

A canonical hard problem for communication complexity is the two-party
disjointness (set intersection) problem, $\bigvee_{i=1}^n x_i\land y_i$ where $x$
and $y$ are indicator vectors of sets in $[n]$.  It has deterministic
communication complexity $n+1$ (and requires $\Omega(n)$ bits be sent even with
randomness, but that is beyond what we need).
We will need a variant of the ``best partition" version of
communication complexity in which the protocol includes a choice of the best
split of input indices $\bf X$ and $\bf Y$ between Alice and Bob. 

A typical method for proving lower bounds on OBDD size for a Boolean function
$f$ begins
by observing that a size $s$ OBDD may be simulated by a $\log s$-bit 
one-way communication protocol where Alice holds the first half of the 
variables read by the OBDD and Bob holds the second half. In this protocol, 
Alice starts at the root of the OBDD and follows the (unique) OBDD path 
determined by her half of the input until she reaches
 a node $v$ querying a variable held by Bob. She then sends the identity of
 the node $v$ to Bob, who can finish the computation starting from $v$. Thus, 
 if we show that $f$ has one-way communication complexity
$CC_{{\bf X} \rightarrow {\bf Y}}(f({\bf X},{\bf Y}))$ at least $C$
 in the best split $\{{\bf X},{\bf Y}\}$ of its input variables,
 then any OBDD computing $f$ must have at least $2^C$ nodes.
 
Our lower bound for SDDs uses related ideas but in a more sophisticated way,
and instead of providing a one-way deterministic protocol, we give an
unambiguous protocol that simulates the SDD computation.  In particular, the
conversion to deterministic protocols requires two-way communication.

\section{SDDs and Best-Partition Communication Complexity}
\label{sdd-lower}

In this section, we show how we can use any small SDD representing a function
$f$ to build an efficient communication protocol for $f$ given an approximately
balanced partition of input variables that is determined by its associated
vtree.  
As a consequence, any function requiring large communication complexity
under all such partitions requires large SDDs.  
To begin this analysis, we consider how an SDD simplifies under a
partial assignment to its input variables.

\subsection{Pruning SDDs Using Restrictions}

\begin{definition}
Suppose that $\vtree$ is a pruned vtree for a set of variables ${\bf X}$, and
that $v$ is a vertex in $\vtree$.
Let $\desc(v)$ denote the set of variables that are descendants of $v$ in
$\vtree$ and $\shell(v) = {\bf X} \setminus \desc(v)$.
Also let $\parent(v)$ denote the (unique) vertex in $\vtree$ that has $v$ as a
child.
\end{definition}

We define a construction to capture what happens to an SDD under a partial
assignment of its variables.

\begin{definition}
Let $\SDD$ be an SDD that respects $\vtree$, a vtree for the variables
${\bf X}$, and suppose that $\SDD$ computes the function $f$.
Let ${\bf B} \subseteq {\bf X}$ and
${\bf A} = {\bf X} \setminus {\bf B}$ and let
$\rho:{\bf A} \rightarrow \{0,1 \}$
be an assignment to the variables in ${\bf A}$. 
Let $\SDD|_{\rho}$ be Boolean circuit remaining after plugging the partial
 assignment $\rho$ into the SDD $\SDD$ and making the following
  simplifications:
  \begin{enumerate}
  \item If a gate computes a constant $c \in \{\top,\bot\}$ under the partial
   assignment $\rho$, we can replace that gate and its outgoing edges with $c$.
  \item Remove any children of OR-gates that compute $\bot$.
  \item Remove any nodes disconnected from the root.
  \end{enumerate}
  
For each vtree vertex $v \in \vtree$ that was not removed in this process,
we denote its counterpart in the pruned vtree $\vtree|_{\bf A}$ by
$v|_{\bf A}$.

Construct the pruned vtree $\vtree|_{{\bf A}}$ from $\vtree$ as follows: for 
each vertex $v$, if $\desc(v) \subseteq {\bf A}$ and
$\desc(\parent(v)) \not\subseteq {\bf A}$, replace $v$ and its subtree by a
stub.
We say that we have \emph{pruned} the subtree rooted at $v$.
(See Figure~\ref{fig:pruning} for an example of an SDD and its vtree both before and after pruning.)
\begin{figure}[t]
  \centering 
\includegraphics[scale = 0.25]{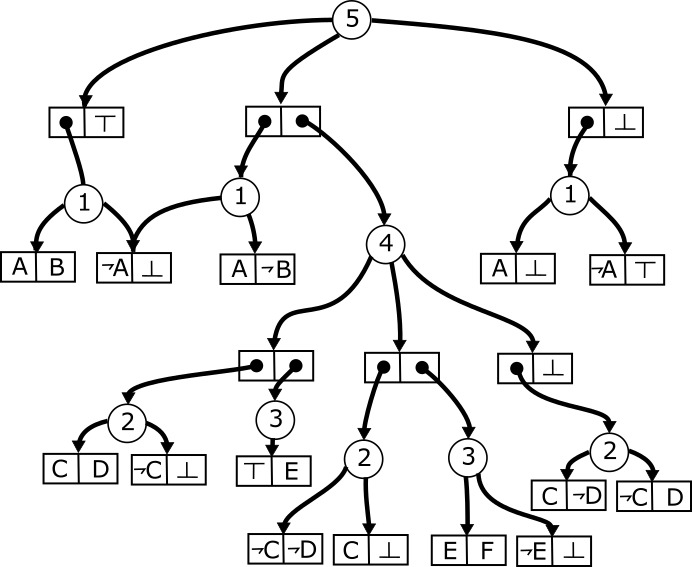}
\hfil 
\includegraphics[scale = 0.25]{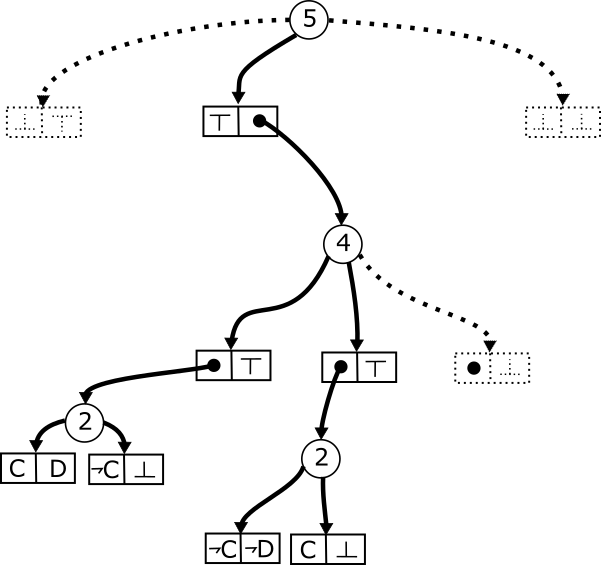}
\vskip 1ex
\includegraphics[scale = 0.25]{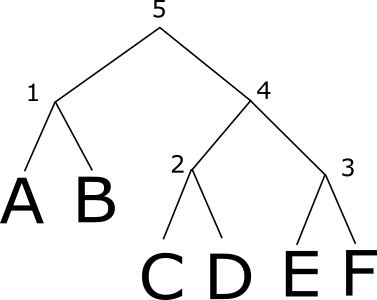}
\hfil 
\includegraphics[scale = 0.25]{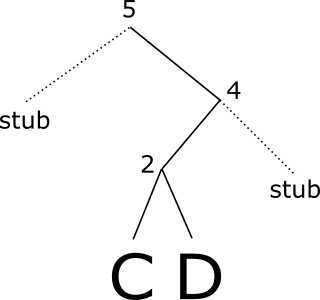}
\caption{An SDD and its vtree, as well as the pruned pair after setting $B$ to 0 and $A,E,F$ to 1.}
\label{fig:pruning}
\end{figure}


For ${\bf A} \subseteq {\bf X}$, we call $\{{\bf A},{\bf X}\setminus {\bf A}\}$
a \emph{shell partition} for ${\bf X}$ if there is a vtree vertex $v \in \vtree$
such that $\shell(v) = {\bf A}$.
We call ${\bf A}$ the \emph{shell}.
If, for a restriction $\rho:{\bf A} \rightarrow \{0,1\}$, there exists a
vtree vertex $v \in \vtree$ such that $\shell(v) = {\bf A}$, we call $\rho$ a
\emph{shell restriction}.
\label{def:restrict}
\end{definition}

\begin{proposition}
\label{properties}
Let $\SDD$ be an SDD that respects $\vtree$, a vtree for the variables
${\bf X}$, and suppose that $\SDD$ computes the function $f$.
Let ${\bf A} \subseteq {\bf X}$ and $\rho:{\bf A} \rightarrow \{0,1 \}$ be
a partial assignment of the variables in ${\bf A}$. The pruned SDD
$\SDD|_{\rho}$ has the following properties:\\
(a) $\langle \SDD|_{\rho} \rangle = f|_{\rho}$.\\
(b) $\SDD|_{\rho}$ is a pruned SDD respecting $\vtree|_{{\bf A}}$.\\
(c) $G(\SDD|_{\rho})$ is a subgraph of $G(\SDD)$.
\end{proposition}

\begin{proof}
(a): An SDD may be equivalently described as a Boolean circuit of alternating
OR and AND gates.
For any Boolean circuit in the variables ${\bf X}$ that computes $f$, plugging
in
the values for the restriction $\rho$ yields a circuit computing $f|_{\rho}$.
Furthermore, the simplification steps do not change the function computed.

(b): For each $v$ such that $\desc(v) \subseteq {\bf A}$ and
$\desc(\parent(v)) \not\subseteq {\bf A}$, we have replaced the subtree rooted
at $v$ by a stub and replaced the SDDs in $\SDD$ respecting $v$ by either
$\top$ or $\bot$.
Thus $\SDD|_{\rho}$ respects $\vtree|_{{\bf A}}$.

We now check that $\SDD|_{\rho}$ is a pruned SDD.
In particular we need to ensure that for each SDD
$\SDD' = \{(p_1,s_1),\ldots,(p_\ell,s_\ell)\}$ in $\SDD$,
the corresponding pruned SDDs that remain from $p_1,\ldots,p_\ell$ in its pruned
counterpart $\SDD'|_{\rho}$ represent a collection of disjoint functions.
From the first part of this proposition, these are
$\langle p_{i_1} \rangle|_{\rho}, \ldots, \langle p_{i_k} \rangle|_{\rho}$
for some $k \leq n$, where we have only included those SDDs that are
consistent under $\rho$.
Since the original set of SDDs was a partition and thus disjoint, this set
of restricted (pruned) SDDs is also disjoint.

(c): The process in Definition~\ref{def:restrict} only removes nodes from
$G(\SDD)$ to construct $G(\SDD|_{\rho})$.
Further, it does not change the label of any SDD that was not removed.
\end{proof}

\subsection{Unambiguous Communication Protocol for SDDs}

The way that we will partition the input variables to an SDD between the parties
Alice and Bob in the communication protocol will respect the structure of its
associated vtree.  The restrictions will correspond to assignments that
reflect Alice's knowledge of the input and will similarly respect that
structure.

Notice that a vtree cut along an edge $(u,v)$ (where $u$ is the parent of $v$)
induces a shell partition for ${\bf X}$ consisting of the set
${\bf B} = \desc(v)$, and the shell ${\bf A} = {\bf X} \setminus {\bf B}$.

\begin{proposition}
\label{disj}
Let $\alpha$ be an SDD of size $s$ computing a function
$f:\set{0,1}^{\bf X} \rightarrow \{0,1\}$ that respects a vtree $\vtree$.
Suppose that $\{{\bf A},{\bf B}\}$ is a shell partition for ${\bf X}$ and that
${\bf A}$ is its shell.
Let $b$ be the vertex in $\vtree$ for which $\desc(b) = {\bf B}$ and
$\desc(\parent(b)) \not\subseteq {\bf B}$.

For any shell restriction $\rho:{\bf A} \rightarrow \{0,1\}$, the set
$\langle \lab_{\SDD|_{\rho}}(b|_{\bf A}) \rangle$ is a disjoint collection
of functions.
\end{proposition}

\begin{sloppypar}
\begin{proof}
For non-shell restrictions $\rho'$, the collection of functions 
$\langle \lab_{\SDD|_{\rho'}}(v)\rangle$ for a vtree node $v$ is not
disjoint; we need to use the specific properties of $\bf A$ and $b$.
Since $\rho$ was a shell restriction, the pruned vtree $\vtree|_{\bf A}$ takes
the form of a path $v_1|_{\bf A}, \ldots, v_{k}|_{\bf A}$ of internal vertices,
where $v_1$ is the root of $\vtree$,  and $v_{k}|_{\bf A}= b|_{{\bf A}}$, with
the other child of each of $v_1|_{\bf A},\ldots, v_{k-1}|_{\bf A}$ being a
stub, together with a vtree for the variables ${\bf B}$ rooted at $b$.
We will show that if $\langle \lab_{\SDD|_{\rho}}( v_i|_{\bf A})\rangle$ is disjoint
then so is $\langle \lab_{\SDD|_{\rho}}( v_{i+1}|_{\bf A})\rangle$.
This will prove the proposition since $\langle \lab_{\SDD|_{\rho}}( v_1|_{\bf A})\rangle$
only contains the function $\langle \SDD|_{\rho}\rangle$ and is therefore
trivially disjoint.

We will use the fact that every pruned-SDD from
$\lab_{\SDD|_{\rho}}( v_{i+1}|_{\bf A})$ is contained in some SDD from
$\lab_{\SDD|_{\rho}}( v_{i}|_{\bf A})$.
We have two cases to check: $v_{i+1}|_{\bf A}$ is either a left child or a
right child of $v_{i}|_{\bf A}$. 

If $v_{i+1}|_{\bf A}$ was a right child then each pruned-SDD
$\eta|_{\rho}$ contained in $\lab_{\SDD|_{\rho}}(v_{i}|_{\bf A})$ takes the form
$\eta|_{\rho} = \{ (\top,s|_{\rho})\}$. 
Then $\langle \lab_{\SDD|_{\rho}}(v_{i+1}|_{\bf A})\rangle = \langle \lab_{\SDD|_{\rho}}(v_{i}|_{\bf A})\rangle$ 
and is therefore disjoint by assumption.

Otherwise suppose that $v_{i+1}|_{\bf A}$ is the left child of $v_{i}|_{\bf A}$.
Let $\eta|_{\rho} \in \lab_{\SDD|_{\rho}}(v_{i}|_{\bf A})$.
Let $\eta|_{\rho} = \{(\eta_1|_{\rho},\top),\ldots,(\eta_k|_{\rho}),\top)\}$
where $\bigvee_{i=1}^{k} \langle \eta_{i}|_{\rho} \rangle = \langle \eta|_{\rho} \rangle$
and $\{\langle \eta_1|_{\rho}\rangle,\ldots,\langle \eta_k|_{\rho} \rangle\}$
, being a collection of primes for $\eta|_{\rho}$,
is disjoint.
By assumption $\langle \lab_{\SDD|_{\rho}}( v_i|_{\bf A})\rangle$ is disjoint, so for any
other $\eta'|_{\rho} = \{(\eta'_1|_{\rho},\top),\ldots,(\eta'_{k'}|_{\rho},\top)\} \in \lab_{\SDD|_{\rho}}(v_{i}|_{\bf A})\}$
distinct from $\eta|_{\rho}$, we have
$\langle\eta|_{\rho} \rangle \wedge \langle\eta'|_{\rho} \rangle = \bot$.
Then for any $i \in [k]$ and $j \in [k']$, we have
$\langle \eta_{i}|_{\rho}\rangle \wedge \langle \eta'_{j}|_{\rho} \rangle = \bot$.
Thus $\langle \lab_{\SDD|_{\rho}}(v_{i+1}|_{\bf A})\rangle$ is disjoint.
\end{proof}
\end{sloppypar}

\begin{theorem}
\label{mainthm}
Let $\alpha$ be an SDD of size $s$  that respects a vtree $\vtree$ and suppose 
that it computes the function $f: \set{0,1}^{\bf X} \rightarrow \{0,1\}.$
Suppose that $\{{\bf A},{\bf B}\}$ is a shell partition for ${\bf X}$ and that
${\bf A}$ is the shell.
Let $b$ be the vertex in $\vtree$ for which $\desc(b) = {\bf B}$ and
$\desc(\parent(b)) \not\subseteq {\bf B}$.

Consider the communication game where Alice has the variables ${\bf A}$, Bob
has the variables ${\bf B}$, and they are trying to compute
$f({\bf A},{\bf B})$.
There is a $\log s$-bit unambiguous communication protocol computing $f$.
\end{theorem}

\begin{proof}
Suppose that Alice and Bob both know the SDD $\alpha$.
Let $\rho:{\bf A} \rightarrow \{0,1\}$ be the partial assignment corresponding
to Alice's input.
This is a shell restriction.
Alice may then privately construct the pruned SDD $\SDD|_{\rho}$, which
computes $f|_{\rho}$ by Proposition~\ref{properties}.
Further,  $\SDD|_{\rho}$ evaluates to $1$ under Bob's input
$\phi:{\bf B} \rightarrow \{0,1\}$ if and only if there exists a
pruned-SDD $\eta|_{\rho} \in \lab_{\SDD|_{\rho}}(b|_{\bf A})$ such that
$\langle \eta|_{\rho} \rangle(\phi) = 1$.

By Proposition~\ref{disj}, $\langle\lab_{\SDD|_{\rho}}(b|_{\bf A})\rangle$ is disjoint.
Also, since $\rho$ is a shell restriction with shell $\bf A$, and 
$\desc(b) = {\bf B}={\bf X}\setminus {\bf A}$, every SDD in
$\lab_{\SDD|_{\rho}}(b|_{\bf A})$ was unchanged by $\rho$.
In particular, this means that $\lab_{\SDD|_{\rho}}(b|_{\bf A})\subseteq \lab_{\SDD}(b)$
and any pruned-SDD $\eta|_\rho$ can be viewed as some
$\eta\in \lab_{\SDD}(b)$ that is also in 
$\lab_{\SDD|_{\rho}}(b|_{\bf A})$.

For the protocol Alice nondeterministically selects an
$\eta$ from $\lab_{\SDD|_{\rho}}(b|_{\bf A})$ and then sends its identity as
a member of $\lab_{\SDD}(b)$ to Bob.
This requires at most $\log s$ bits.
Bob will output 1 on his input $\phi$ if and only if
$\langle \eta\rangle (\phi)=1$, which he can test since he knows $\SDD$ and
$b$.
This protocol is unambiguous since the fact that
$\langle\lab_{\SDD|_{\rho}}(b|_{\bf A})\rangle$ is disjoint means
means that for any input $\phi$ to Bob there is
at most one $\eta\in \lab_{\SDD|_{\rho}}(b|_{\bf A})$ such that 
$\langle \eta\rangle(\phi)=1$.
Since Bob knows $\SDD$, he also knows $\eta$ and can therefore compute
$\langle \eta \rangle(\phi)$.
Since $\SDD$ computes $f$, if $\langle \eta \rangle(\phi) = 1$ then
$f(\phi,\rho) = 1$.
Otherwise all of the functions in $\langle \lab_{\SDD|_{\rho}}(b|_{\bf A}) \rangle$
evaluate to $0$ on input $\phi$ so $f(\phi,\rho) = 0$.
\end{proof}

We can relate the deterministic and unambiguous communication complexities of a
function using the following result from~\cite{yannakakis:extension}.
We include a proof of this result in the appendix for completeness.

\begin{theorem}[Yannakakis]
\label{yan}
If there is an $g$-bit unambiguous communication protocol for a function
$f:\set{0,1}^{\bf A}\times \set{0,1}^{\bf B} \rightarrow \{0,1\}$,
then there is a $(g+1)^2$-bit deterministic protocol for $f$.
\end{theorem}

The following $1/3$-$2/3$ lemma is standard.

\begin{lemma}
For a vtree $\vtree$ for $L$ variables, if a vertex $b$ satisfies
$\frac{1}{3} L \leq |\desc(b)| \leq \frac{2}{3} L$, we call it a
\emph{$(1/3,2/3)$ vertex}.
Every vtree contains a $(1/3,2/3)$ vertex.
\label{balnode}
\end{lemma}

\cut{
\begin{proof}
Starting from the root, we will choose to go to the left or right child until
we reach a $(1/3,2/3)$ vertex.
If we are at a vertex $v \in \vtree$ where $|\desc(v)| > \frac{2}{3} L$, then we
choose to move to the child $v'$ with a greater number of variables in its
subtree.
This child $v'$ always has at least $\frac{1}{3} L$ variables in its subtree
(otherwise there are fewer than $\frac{2}{3} L$ variables in the subtree of
$v$).
Since $|\desc(v')| < |\desc(v)|$ for any $v' \in \desc(v)$, we will eventually
reach a $(1/3,2/3)$ vertex.
\end{proof}
}

\begin{definition}
Let ${\bf X}$ be a set of variables and $({\bf A,B})$ a partition of ${\bf X}$.
We call the partition $({\bf A,B})$ a \emph{$(\delta,1-\delta)$-partition} for
$\delta \in [0,1/2]$ if $\min(|{\bf A}|,|{\bf B}|) \geq \delta |{\bf X}|$.
That is, the minimum size of one side of the partition is at least a
$\delta$-fraction of the total number of variables.

The best $(\delta,1-\delta)$-partition communication complexity of a Boolean
function $f:\set{0,1}^{\bf X} \rightarrow \{0,1\}$ is $\min( CC(f({\bf A,B})))$ where
the minimum is taken over all $(\delta,1-\delta)$-partitions $({\bf A,B})$.
\end{definition}

\begin{theorem}
\label{maincor}
If the best $(1/3,2/3)$-partition communication complexity of a Boolean
function $f:\set{0,1}^{\bf X} \rightarrow \{0,1\}$ is $C$, then an SDD computing $f$
has size at least $2^{\sqrt{C}-1}$.
\end{theorem}

\begin{proof}
Suppose that $\SDD$ is an SDD of size $s$ respecting the vtree $\vtree$ for
variables ${\bf X}$, and that $\SDD$ computes $f$.
From Lemma~\ref{balnode} the vtree $\vtree$ contains a $(1/3,2/3)$ vertex $b$.
This $(1/3,2/3)$ vertex $b$ induces a $(1/3,2/3)$-partition of the variables
$\{{\bf A},{\bf B}\}$ where ${\bf B} = \desc(b)$ and ${\bf A} = \shell(b)$.
Further, this partition $\{{\bf A},{\bf B}\}$ is a shell partition.
By Theorem~\ref{mainthm}, there exists a $\log s$-bit unambiguous communication
protocol for $f({\bf A},{\bf B})$.
Then by Theorem~\ref{yan}, there exists a $(\log(s)+1)^2$-bit deterministic
communication protocol for $f({\bf A},{\bf B})$.
Since the best $(1/3,2/3)$-partition communication complexity of $f$ is $C$,
we have that $C \leq  (\log(s)+1)^2$ which implies that
$s \geq 2^{\sqrt{C}-1}$ as stated.
\end{proof}

\section{Lower Bounds for SDDs}
\label{applications}

There are a large number of predicates $f:\set{0,1}^n\rightarrow \set{0,1}$
for which the $(1/3,2/3)$-partition communication complexity is $\Omega(n)$
and by Theorem~\ref{maincor} each of these requires SDD size
$2^{\Omega(\sqrt{n})}$.
The usual best-partition communication complexity is $(1/2,1/2)$-partition
communication complexity.
For example, the function {\sc ShiftedEQ} which takes as inputs $x,y\in \set{0,1}^n$ and $z\in \set{0,1}^{\lceil \log_2 n\rceil}$ and tests whether or not
$y=SHIFT(x,z)$ where $SHIFT(x,z)$ is the cyclic shift of $x$ by $(z)_2$
positions.
However, as is typical of these functions, the same proof which shows that the 
$(1/2,1/2)$-partition communication complexity of {\sc ShiftedEQ} is
$\Omega(n)$ also shows that its
$(1/3,2/3)$-partition communication complexity is $\Omega(n)$.
However, most of these functions are not typical of predicates to which
one might want to apply weighted model counting.  
Instead we analyze SDDs for formulas derived from a natural class of database
queries.
We are able to characterize SDD size for these queries, proving exponential
lower bounds for every such query that cannot already be represented in linear
size by an OBDD.
This includes an example of a query called $Q_V$ for which FBDDs are polynomial
size but the best SDD requires exponential size.

\subsection{SDD Knowledge Compilation for Database Query Lineages}

We analyze SDDs for a natural class of database queries called the {\em union
of conjunctive queries (UCQ)}.  This includes all queries given by the grammar
$$q::=R({\bf x})\mid \exists x q\mid q\land q\mid q \lor q$$
where $R(\bf x)$ is an elementary relation and $x$ is a variable.
For each such query $q$, given an input database $D$, the query's 
\emph{lineage}, $\Phi^D_q$, is a Boolean expression for $q$ over Boolean 
variables that correspond to tuples in $D$. 
In general, one thinks of the query size as fixed and considers the complexity
of query evaluation as a function of the size of the database. 
The following formulas are lineages (or parts thereof) of
well-known queries 
that that are fundamental for probabilistic 
databases~\cite{DBLP:journals/jacm/DalviS12,DBLP:journals/mst/JhaS13}
over a particular database $D_0$ (called the {\em complete bipartite graph of
size $m$} in \cite{DBLP:journals/mst/JhaS13}):
\begin{align*}
H_0 &= \bigvee\limits_{i,j \in [m]} R_i S_{ij} T_j\\
Q_V &= \bigvee\limits_{i,j \in [m]} R_i S_{ij} \vee S_{ij} T_j \vee R_i T_j \\
H_1 &= \bigvee\limits_{i,j \in [m]} R_i S_{ij} \vee S_{ij} T_j\\
H_{k0} &= \bigvee\limits_{i \in [m]} R_i S^1_{ij} \qquad\mbox{for $k\ge 1$}\\
H_{k\ell} &= \bigvee\limits_{i,j \in [m]} S^\ell_{ij}S^{\ell+1}_{ij} \qquad\mbox{for $0<\ell<k$}\\
H_{kk} &= \bigvee\limits_{i \in [m]} S^k_{ij} T_j \qquad\mbox{for $k\ge 1$.}
\end{align*}
(The corresponding queries are represented using lower case letters
$h_0, q_V, h_1, h_{k0},\ldots, h_{kk}$ and involve unary relations $R$ and $T$, as well as binary relations $S$ and $S^k$. For example, $h_0=\exists x_0\exists y_0  R(x_0) S(x_0,y_0) T(y_0)$.)
The following lemma will be useful in identifying subformulas of the above
query lineages that can be used to compute the set disjointness function.

\begin{proposition}
Let the elements of $[m]\times[m]$ be partitioned into two sets $A$ and $B$,
each of size at least $\delta m^2$.
Let $\row(i)$ denote $\set{i}\times [m]$ and
$\col(j)$ denote $[m]\times \set{j}$.
Define $W_{\row}=\set{i\in [m]\mid \emptyset \ne \row(i)\cap A\mbox{ and }\emptyset\ne \row(i)\cap B}$.
That is, $\row(i)$ for $i \in W_{\row}$ is split into two nonempty pieces by
the partition.
Similarly, define $W_{\col}=\set{i\in [m]\mid \emptyset \ne \col(j)\cap A\mbox{ and }\emptyset\ne \col(j)\cap B}$.
Then\\
\centerline{$\max(|W_{\row}|,|W_{\col}|) \geq \sqrt{\delta}\cdot m.$}
\label{W_bound}
\end{proposition}
\vspace*{-0.3in}
\begin{proof}
Suppose that both $|W_{\row}| < m$ and $|W_{\col}| < m$.
By definition, if $i\notin W_{\row}$ then one of $A$ or $B$ contains an
entire row, $\row(i)$, say $A$ without loss of generality.
This implies that no column $\col(j)$ is entirely contained in
$B$.
Since $|W_{\col}|<m$, there is some column $\col(j)$ that is entirely contained
in $A$.
This in turn implies that $B$ does not contain any full row.
In particular, we have that $A$ contains all rows in $[m]\setminus W_{\row}$
and all columns in $[m]\setminus W_{\col}$ and thus
$B\subseteq W_{\row}\times W_{\col}$ and so $|B|\le |W_{\row}|\cdot |W_{\col}|$.
By assumption, $|B| \geq \delta m^2$.
Hence $|W_{\row}| \cdot |W_{\col}|\ge \delta m^2$ and
so $\max\{|w_{\row}|,|w_{\col}|\}\ge \sqrt{\delta}\cdot m$.
\end{proof}

\begin{theorem}
\label{QV-H1}
For $m \geq 6$, the best $(1/3,2/3)$-partition communication complexity of
$Q_V$, $H_0$, and of $H_1$ is at least $m/3$.
\label{qv}
\end{theorem}
\vspace*{-3ex}
\begin{proof}
Let $\bf X$ be the set of variables appearing in $Q_V$ (or $H_1$) and
let $({\bf A},{\bf B})$ be a $(1/3,2/3)$-partition of $\bf X$.
Let $(A,B)$ be the partition of $[m]\times [m]$ induced by $({\bf A},{\bf B})$
and define $W_{\row}$ and $W_{\col}$ as in Proposition~\ref{W_bound}.
Since $|{\bf X}|=m^2+2m$ and only elements of $[m]\times [m]$ are relevant,
$|A|,|B|\ge (m^2+2m)/3 -2m=(1-4/m)m^2/3\ge m^2/9$ for
$m\ge 6$ and hence $\max(|W_{\row}|,|W_{\col}|)\ge m/3$.
We complete the proof by showing that computing $Q_V({\bf A},{\bf B})$
and $H_1({\bf A},{\bf B})$ each
require at least
$\max(|W_{\row}|,|W_{\col}|)$ bits of communication between Alice and Bob.
We will do this by showing that for a particular subset of inputs, $Q_V$ is
equivalent to the disjointness function for a $\max(|W_{\row}|,|W_{\col}|)$
size set.

Suppose without loss of generality that $|W_{\row}| \geq |W_{\col}|$.
Set all $T_j=0$ and for each $i\notin W_{\row}$ set $R_i=0$.
For each $i \in W_{\row}$ for which $R_i \in {\bf A}$, set
all $S_{ij} \in {\bf A}$ to $0$, let $j_i$ be minimal such that
$S_{ij_i}\in {\bf B}$, and set $S_{ij}\in {\bf B}$ to $0$ for all $j>j_i$. 
(Such an index $j_i$ must exist since $i\in W_{\row}$.)
Similarly, For each $i \in W_{\row}$ for which $R_i \in {\bf B}$, set
all $S_{ij} \in {\bf B}$ to $0$, let $j_i$ be minimal such that
$S_{ij_i}\in {\bf A}$, and set $S_{ij}\in {\bf A}$ to $0$ for all $j>j_i$. 
In particular, under this partial assignment, we have\\[-3ex]
\begin{align*}
&Q_V = H_1= \bigvee_{i\in W_{\row}} R_i S_{ij_i}\\[-5ex]
\end{align*}
and for each $i\in W_{\row}$,
Alice holds one of $R_i$ or $S_{ij_i}$ and Bob holds the other.
We can reduce $H_0$ to the same quantity by setting all $T_j=1$.
This is precisely the set disjointness problem on two sets of size $|W_{\row}|$
where membership of $i$ in each player's set is determined by the value of
the unset bit indexed by $i$ that player holds.
Therefore, computing $Q_V$ or $H_1$ requires at
least $|W_{\row}|$ bits of communication, as desired.
\end{proof}

Combining this with Theorem~\ref{maincor}, we immediately obtain the following:

\begin{theorem}
For $m\ge 6$, any SDD representing $Q_V$ or $H_1$ requires size at least
$2^{\sqrt{m/3}-1}$.
\end{theorem}

As \cite{DBLP:journals/mst/JhaS13} has shown that $Q_V$ has FBDD size $O(m^2)$, we
obtain the following separation.

\begin{corollary}
\label{FBDD<SDD}
FBDDs can be exponentially more succinct than SDDs.  In particular, $Q_V$ has
FBDD size $O(m^2)$ but every SDD for $Q_V$ requires size $2^{\sqrt{m/3}-1}$
for $m\ge 6$.
\end{corollary}

We now consider the formulas $H_{ki}$ above.  Though they seem somewhat
specialized, these formulas are fundamental to UCQ queries:
\cite{DBLP:journals/mst/JhaS13} define the notion of an {\em inversion} in
a UCQ query and use it to characterize the OBDD size of UCQ queries.
In particular they show that if a query $q$ is {\em inversion-free} then the
OBDD size of its lineage $Q$ is linear and if $q$ has an minimum inversion
length $k\ge 1$ then it requires OBDD size $2^{\Omega(n/k)}$ where
$n$ is the domain size of all attributes.
Jha and Suciu obtain this lower bound by analyzing the $H_{ki}$ we defined 
above. 
(We will not define the notion of inversions, or their lengths,
and instead use the definition as a black box. However, as an example, the query associated with
$H_1$ has an inversion of length 1 so its OBDD size is $2^{\Omega(m)}$.)

\begin{proposition}\cite{DBLP:journals/mst/JhaS13}
\label{restrict}
Let $q$ be a query with a length $k \geq 1$ inversion.
Let $D_0$ be the complete bipartite graph of size $m$.
There exists a database $D$ for $q$, along with variable restrictions
$\rho_i$ for all $i \in [0,k]$, such that $|D| = O(|D_0|)$ and
$\Phi^D_q|_{\rho_i} = \Phi^{D_0}_{h_{ki}}=H_{ki}$
\end{proposition}

\vspace*{0.5ex}

\begin{theorem}
Let $k\ge 2$ and assume that $m\ge 6$.
Let $q$ be a query with a length $k \geq 2$ inversion. 
Then there exists a database $D$ for which any SDD for $Q=\Phi^D_{q}$
has size at least $2^{\sqrt{m/k}/3-1}$.
\end{theorem}
\vspace*{-2.5ex}
\begin{proof}
Given a query $q$, let $D$ be the database for $q$ constructed in
Proposition~\ref{restrict}.
Fix the vtree $\vtree$ over ${\bf X}_k$ respected by an SDD $\SDD$ for 
$\Phi^D_q$.
By Lemma~\ref{balnode}, there exists a $(1/3,2/3)$ node $b$ in the vtree
$\vtree$ 
that gives a $(1/3,2/3)$ partition $\{{\bf A},{\bf B}\}$ of ${\bf X}_k$. 
By Proposition~\ref{restrict}, there are restrictions $\rho_0,\ldots,\rho_k$
such that $\Phi^D_q|_{\rho_i}=H_{ki}$ for all $i$.  
Thus $\alpha|_{\rho_i}$ is a (pruned) SDD, of
size $\le$ that of $\SDD$, respecting $\vtree|_{\rho_i}$ and computing $H_{ki}$.
Observe that the restriction of $\{{\bf A},{\bf B}\}$ to the variables
of ${\bf X}_{ki}$ is also shell partition of $\vtree|_{\rho_i}$ at node $b$.

We will show that there must exist an $H_{ki}$ for which 
$\CC(H_{ki}({\bf A},{\bf B})) \geq m/(9k)$ and therefore by
Theorem~\ref{yan}, this implies that the unambiguous
communication complexity of $H_{ki}$ is at least $\frac13\sqrt{m/k}-1$
Then by Theorem~\ref{mainthm}, any SDD respecting $\vtree$ that computes 
$H_{ki}$ has size at least $2^{\frac13\sqrt{m/k}-1}$.

Let $W_{\stack}$ contain all pairs $(i,j)$ for which both
${\bf A}\cap \bigcup_{\ell=1}^k \set{S^\ell_{ij}}\ne \emptyset$ and 
${\bf B}\cap \bigcup_{\ell=1}^k \set{S^\ell_{ij}}\ne \emptyset$ and 
Let $\gamma=1/9$.
We will consider two cases: either 
$|W_{\stack}| \geq \gamma\cdot m$ or $|W_{\stack}| < \gamma\cdot m$.

In the first case, since $|W_{\stack}| \geq \gamma\cdot m$, there must exist
at least $\gamma\cdot m$ tuples $(i,j,\ell)$ for which either $S^\ell_{ij}
\in {\bf A}$ and $S^{\ell+1}_{ij} \in {\bf B}$ or vice-versa.
Call the set of these tuples ${\bf T}$. Then, since there are $k-1$ choices of
$\ell<k$, there exists some $\ell^*$ such 
that the set ${\bf T}_{\ell^*} := {\bf T} \cap [m]\times[m]\times \set{\ell^*}$
contains at least $\gamma\cdot m/(k-1)> m/(9k)$ elements. 
If we set all variables of ${\bf X}_{k\ell^*}$ outside of
${\bf T}_{\ell^*}$ to $0$, the function $H_{k\ell^*}$ corresponds to solving a
disjointness problem between Alice and Bob on the elements of
${\bf T}_{\ell^*}$. 
Thus the communication complexity of $H_{k\ell^*}$ under the 
 partition $\{ {\bf A},{\bf B}\}$ is at least $m/(9k)$.

In the second case, consider the largest square submatrix $M$ of
$[m]\times [m]$ that does not contain any member of $W_{\stack}$.  
We mimic the argument of Theorem~\ref{QV-H1} on this submatrix $M$.
By definition, $M$ has side $m'\ge (1-\gamma)m$.
For every $(i,j)$ in $M$, either ${\bf A}$ or ${\bf B}$ contains all
$S^\ell_{ij}$; let $A$ be those $(i,j)$ such that these are in ${\bf A}$ and
$B$ be those $(i,j)$ for which they are in ${\bf B}$.
Since $|{\bf A}|,|{\bf B}|\ge|{\bf X}_k|/3=(km^2+2m)/3$ and there are at most
$2m+(\gamma^2+2\gamma) km^2$ variables not in $M$,\\[-3ex]
\begin{align*}
|A|,|B|&\ge [(km^2+2m)/3-2m+(\gamma^2+2\gamma) km^2]/k\\
&=[(1-\gamma)^2-2/3-4/(3km)]m^2> (m/18)^2)\\[-4ex]
\end{align*}
since $k\ge 2$.
Applying Proposition~\ref{W_bound}, we see that
$\max(|W_{\row}|,|W_{\col}|)\ge m/18\ge m/(9k)$.
By the same argument presented in the proof of Theorem~\ref{qv}, we have both
$\CC(H_{k0}({\bf A},{\bf B})) \geq  |W_{\row}|$ and
$\CC(H_{kk}({\bf A},{\bf B})) \geq  |W_{\col}|$
so at least one of these is at least $m/(9k)$ and the theorem follows.
\end{proof}
\vspace*{-2ex}
It follows that for inversion-free UCQ queries, both SDD and OBDD sizes of
any lineage are linear, while UCQ queries with inversions (of length $k$)
have worse-case lineage size that is exponential ($2^{\Omega(m/k)}$ for OBDDs
and $2^{\Omega(\sqrt{m/k})}$ for SDDs).
Note that the same SDD size lower bound for UCQ query lineage $Q=\Phi^D_q$ 
applies to its dual $Q^*=\Phi^D_{q^*}$ as follows:  Flipping the signs
on the variables in $Q^*$ yields a function equivalent to $\lnot Q$.
So flipping the variable signs at the leaves of an SDD for $Q^*$ we obtain an
SDD of the same size for $\lnot Q$ and hence a deteministic protocol that also
can compute $Q$.

\section{Simulating DNNFs by OR-FBDDs}
\label{dnnf-simulation}

In this section, we extend the simulation of decision-DNNFs by FBDDs 
from~\cite{Beame+13-uai} to obtain a simulation of general DNNFs
by OR-FBDDs with at most a quasipolynomial increase in size. This simulation 
yields lower bounds on DNNF size from OR-FBDD lower bounds. 
This simulation is also tight, since ~\cite{DBLP:journals/corr/Razgon15a,DBLP:conf/kr/Razgon14}
has shown a quasipolynomial separation between the sizes of DNNFs and OR-FBDDs.

\begin{definition}
For each AND node $u$ in a DNNF $\DNNF$, let $M_u$ be the number of AND nodes 
in the subgraph $D_u$. We call $u$'s left child $u_l$ and its right child $u_r$.
We will assume $M_{u_l} \leq M_{u_r}$ (otherwise we swap $u_l$ and $u_r$).

For each AND node $u$, we classify the edge $(u,u_l)$ as a \emph{light edge}
 and the edge $(u,u_r)$ a \emph{heavy edge}. 
 We classify every other edge in $\DNNF$ as a \emph{neutral edge.}

For a DNNF $\DNNF$ or an OR-FBDD $\mathcal{F}$, we denote the functions that
$\DNNF$ and $\mathcal{F}$ compute as $\Phi_{\DNNF}$ and $\Phi_{\mathcal{F}}$.
\end{definition}

\subsection*{Constructing the OR-FBDD}
For a DNNF $\DNNF$, we will treat a leaf labeled by the variable $X$ as a 
decision node that points to a $0$-sink node if $X = 0$ and a $1$-sink node
if $X=1$, and vice-versa for a leaf labeled by $\neg X$. We also assume that
each AND node has just two children, which only affects the DNNF size by at
 most polynomially.

\begin{definition}
Fix a DNNF $\DNNF$. 
For a node $u$ in $\DNNF$ and a path $P$ from the root to $u$, let $S(P)$ be
the set of light edges along $P$ and 
$S(u)=\{S(P) \mid \textrm{P is a path from the root to $u$}\}$.

We will construct an OR-FBDD $\mathcal{F}$ that computes the same boolean
function as $\DNNF$. Its nodes are pairs $(u,s)$ where $u$ is a node 
in $\DNNF$ and the set of light edges $s$ belongs to $S(u)$. 
Its root is $(\textrm{root}(\DNNF), \emptyset).$ 
The edges in $\mathcal{F}$ are of three types:

Type 1: For each light edge $e=(u,v)$ in $\DNNF$ and $s \in S(u)$, add the
edge $((u,s),(v,s \cup \{e\}))$ to $\mathcal{F}$.

Type 2: For each neutral edge $e=(u,v)$ in $\DNNF$ and $s \in S(u)$, add the
edge $((u,s),(v,s))$ to $\mathcal{F}$.

Type 3: For each heavy edge $(u,v_r)$, let $e=(u,v_l)$ be its sibling light 
edge. 
For each $s \in S(u)$ and 1-sink node $w$ in $D_{v_l}$, add the 
edge $((w,s \cup \{e\}),(v_r,s))$ to $\mathcal{F}$.

We label the nodes $u' = (u,s)$ as follows: (1) if u is a decision node
in $\DNNF$ for the variable $X$ then $u'$ is a decision node in $\mathcal{F}$
testing the same variable $X$, (2) if $u$ is an AND-node, then $u'$ is a
no-op node, (3) if $u$ is an OR node it remains an OR node.
(4) if $u$ is a 0-sink node, then $u'$ is a 0-sink node, (5) if $u$
is a 1-sink node, then: if $s= \emptyset$ then $u'$ is a 1-sink node, otherwise
it is a no-op node. 
\end{definition}
We show an example of this construction in Figure~\ref{fig:dnnf}.
\begin{figure}[t]
  \centering 
\includegraphics[scale = 0.4]{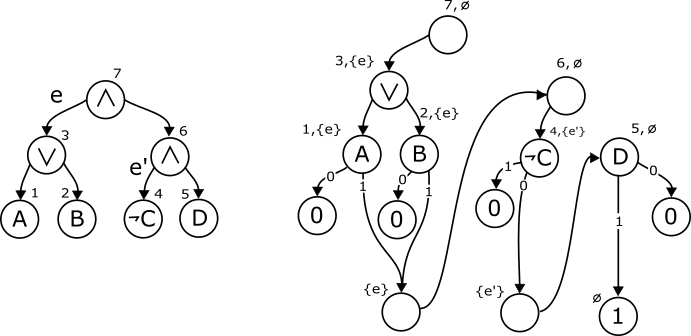}
\caption{A DNNF and our construction of an equivalent OR-FBDD.}
\label{fig:dnnf}
\vspace{-3.1ex}
\end{figure}

\subsection*{Size and Correctness}
\begin{lemma}
For the DNNF $\DNNF$ let $L$ denote the maximum number of light edges from the
root to a leaf, $M$ the number of AND nodes and N the total number of nodes.
Then $\mathcal{F}$ has at most $NM^L$ nodes.
Further, this is $N\cdot 2^{\log^2 N}$.
\end{lemma}

\begin{proof} 
The nodes in $\mathcal{F}$ are labeled $(u,s)$.
There are $N$ possible nodes $u$ and at most $M^L$ choices for the set $s$,
as each path to $u$ has at most $L$ light edges.

Consider a root to leaf path with $L$ light edges.
As we traverse this path, every time we cross a light edge, we decrease the
number of descendant AND nodes by more than half.
Thus we must have begun with more than $2^L$ descendant AND nodes at the root
so that $N\ge M > 2^L$.
This implies that $NM^L$ is quasipolynomial in $N$,

This upper bound is quasipolynomial in $N$, we will show that $M > 2^L$.
Then, since $N \geq M$, $NM^L \leq N 2^{\log^2 M} \leq N 2^{\log^2N}$.
\end{proof}

The proof of the following lemma is in the full paper.

\begin{lemma}
\label{correct-convert}
$\mathcal{F}$ is a correct OR-FBDD with no-op nodes that computes the
same function as $\DNNF$.
\end{lemma}

Using the quasipolynomial simulation of DNNFs by OR-FBDDs, we obtain DNNF 
lower bounds from OR-FBDD lower bounds.

\begin{definition}
%
Function $\operatorname{PERM}_n$ takes an 
$n \times n$ boolean matrix $M$ as input and outputs $1$ if and only if $M$ is 
a permutation matrix.
The function $\operatorname{ROW-COL}_n$ takes an $n \times n$ boolean matrix 
$M$ as input and outputs $1$ if and only if $M$ has an all-$0$ row or an 
all-$0$ column.
\end{definition}

\begin{theorem}
Any OR-FBDD computing 
$\operatorname{PERM}_n$ or $\operatorname{ROW-COL}$, must have size $2^{\Omega(n)}$~\cite{wegener2000book}.
\end{theorem}

\cut{
\begin{corollary}
Any OR-FBDD computing $\operatorname{ROW-COL}$ must have size $2^{\Omega(n)}$. 
\end{corollary}

\begin{proof}
Given an OR-FBDD $\mathcal{F}$ of size $s$ computing 
$\operatorname{ROW-COL}_n$, we will find an OR-FBDD of size $(n+2)s$ 
computing $\operatorname{PERM}_n$. Let $V_{\mathcal{F}}$ denote the set of nodes for $\mathcal{F}$.

An $n \times n$ boolean matrix $M$ is a permutation matrix if and only if it 
contains $n$ $1$ entries and has no all-$0$ row or column. We construct an OR-FBDD 
$\mathcal{F}'$ for $\operatorname{PERM}_n$ as follows: the nodes of 
$\mathcal{F}'$ are the set $V_{\mathcal{F}} \times [0,n+1]$. 
This splits each node $v$ from $\mathcal{F}$ into the $n+2$ pieces 
$v_0, \ldots, v_{n+1}$ where each $v_i$ represents that we have seen $i$ 
$1$-entries for $i \leq n$ and $v_{n+1}$ represents that we have seen more than 
$n$ $1$-entries and is a $0$-sink node. 
A node $v_i \in V_{\mathcal{F}'}$ is accepting if and only if $v$ was a 
$0$-sink in $\mathcal{F}$ and $i = n$.
\end{proof}
}

\begin{corollary}
Any DNNF computing 
$\operatorname{PERM}_n$ or $\operatorname{ROW-COL}$
has size at least $2^{\Omega(\sqrt{n})}$
\end{corollary}

\section{Discussion}
\label{discussion}

We have made the first significant progress in understanding the
complexity of general DNNF representations.  We have also provided a new
connection between SDD representations and best-partition communication
complexity.   Best-partition communication complexity is a
standard technique used to derive lower bounds on OBDD size, where it often
yields asymptotically tight results. For communication lower bound $C$,
the lower bound for OBDD size is $2^{C}$ and the lower bound we have shown for
SDD size is $2^{\sqrt C}-1$.  
This is a quasipolynomial difference and matches 
the quasipolynomial separation between OBDD and SDD size shown in~\cite{DBLP:conf/kr/Razgon14}.
Is there always a quasipolynomial simulation
of SDDs by OBDDs in general, matching the quasipolynomial simulation
of decision-DNNFs by FBDDs?  Our separation result shows an example for which
SDDs are sometimes exponentially less concise than FBDDs, and hence
decision-DNNFs also.  Are SDDs ever more concise than decision-DNNFs?

By plugging in the arguments of
\cite{DBLP:conf/aaai/PipatsrisawatD10,pipatsrisawat:dissertation} in place
of Theorem~\ref{mainthm}, all of our lower bounds immediately extend to size
lower bounds for structured deterministic DNNFs (d-DNNFs), of which SDDs are
a special case.
It remains open whether structured d-DNNFs are strictly more concise than SDDs.
\cite{pipatsrisawat08compilation,pipatsrisawat:dissertation}
have proved an exponential separation between
structured d-DNNFs and OBDDs using the {\em Indirect Storage Access (ISA)}
function~\cite{DBLP:journals/tcs/BreitbartHR95}, but the small structured d-DNNF for this function is very far from
an SDD.
It is immediate that, under any variable partition, the $ISA_n$ function has an
$O(\log n)$-bit two-round deterministic communication protocol.  On the
other hand, efficient
one-round (i.e., one-way) communication protocols yield small OBDDs so 
there are two possibilities if SDDs and structured d-DNNFs have different power.
Either 
(1) communication complexity considerations on their own are not enough to
derive a separation between SDDs and structured d-DNNFs,
or 
(2) every SDD can be simulated by an efficient one-way communication protocol,
in which case SDDs can be simulated efficiently by OBDDs (though the ordering
cannot be the same as the natural traversal of the associated
vtree, as shown by~\cite{XueChoiDarwiche12}).

\section*{Acknowledgements}
We thank Dan Suciu and Guy Van den Broeck for helpful comments and suggestions.
\bibliographystyle{apalike}
\bibliography{theory,bib} 
\appendix
\section{Proof of Theorem~\ref{yan}}

Let $f({\bf A},{\bf B})$ be a function with unambiguous communication
complexity $g$ and let $M_f$ be its communication matrix.
Then there exists a set $D$ of $2^g$ disjoint monochromatic rectangles that
cover the $1$'s of $M_f$.

Let $G$ be a graph whose nodes are the rectangles in $D$ and which has an edge
connecting two rectangles if they share some row of $M_f$.
Then each row $r$ of $M_f$ corresponds to a clique $K_r$ containing the
rectangles intersecting $r$.
Similarly, every column $c$ corresponds to an independent set $I_c$ containing
the rectangles intersecting $c$.
For each row $r$ and column $c$, the corresponding entry $M_f(r,c)$ is $1$
if and only if $K_r \cap I_c \not = \emptyset$.
Thus for proving the theorem, it suffices to give a $g^2$ deterministic
protocol for solving the Clique vs Independent set problem on a graph $G$ with
$2^g$ vertices.

The protocol reduces the graph in each step.
Suppose that Alice holds a clique $K$ of an $n$ vertex graph $G$ and Bob
holds an independent set $I$.
In each round Alice sends a node $u \in K$ that is adjacent to fewer than half
the nodes of $G$, or if no such node exists, she notifies Bob. 

If Alice sent the node $u$, then Bob responds with whether (i) $u \in I$,
in which case $K \cap I \not= \emptyset$, or (ii) that $u$ is not adjacent to
any node of $I$, in which case $K \cap I = \emptyset$.
If neither (i) nor (ii) occur then the nodes not adjacent to $u$ are removed
from $G$ as they cannot be in $K$ and the protocol repeats.

Otherwise, if every $u \in K$ is adjacent to over half the nodes of $G$,
Bob sends a node $v \in I$ that is adjacent to at least half the nodes in $G$
if such a $v$ exists.
In this case Alice tells Bob that (i) $v \in K$ so that
$K \cap I \not= \emptyset$, or (ii) $v$ is adjacent to all nodes in $K$ so that
$K \cap I = \emptyset$.
Otherwise, Bob says he has no such $v \in I$ and the nodes adjacent to $v$ are
removed from $G$ and the protocol repeats.

Each iteration of this protocol removes at least half the nodes so that there
are at most $g$ iterations.
The communication per iteration is at most $g+1$ (to either send one of $2^g$
nodes or that no good node exists).

\section{Proof of Lemma~\ref{correct-convert}}

\begin{lemma}
$\mathcal{F}$ is a correct OR-FBDD with no-op nodes.
\end{lemma}

\begin{proof}We need to show that $\mathcal{F}$ is acyclic and that every 
path reads a variable at most once.
These two properties follow from the lemma:

\begin{lemma}
If $u$ is a leaf node in $\DNNF$ labeled by the variable $X$ and there
exists a non-trivial path (with at least one edge) between the nodes
$(u,s), (v,s')$ in $\mathcal{F}$, then the variable $X$ does not occur
in $\DNNF_v$.
\end{lemma}

This lemma implies that $\mathcal{F}$ is acyclic: a cycle in $\mathcal{F}$
implies a non-trivial path from some node $(u,s)$ to itself, but $X \in \DNNF_u$.
It also implies that every path in $\mathcal{F}$ is read-once: if a path tests
a variable $X$ twice, first at $(u,s)$ and again at $(u_1,s_1)$,
then $X \in D_{u_1}$ contradicting the claim.

To prove the lemma, suppose to the contrary that there exists an OR-FBDD node 
$(u,s)$ 
such that $u$ is a leaf labeled with $X$ and that there exists a path from $(u,s)$ to 
$(v,s')$ in $\mathcal{F}$ such that $X$ occurs in $\DNNF_v$. 
Choose $v$ such that $\DNNF$ is maximal; i.e. there is no path from $(u,s)$ to 
some $(v',s'')$ such that $\DNNF_v \subset \DNNF_{v'}$ and $X$ occurs 
in $\DNNF_{v'}$. 
Consider the last edge on the path from $(u,s)$ to $(v,s')$ in $\mathcal{F}$:
\[ (u,s),...,(w,s''),(v,s').\]

Observe that $(w,v)$ is not an edge in $\DNNF$ since $\DNNF_v$ is maximal, 
and $(u,v)$ is not an edge in $\DNNF$ since $u$ was a leaf. 
Therefore the edge from $(w,s'')$ to $(v,s')$ is Type 3. 
So $\DNNF$ has an AND-node $z$ with children $v_l,v$ and the last path edge 
is of the form $(w,s' \cup \{e\}),(v,s')$ where $e=(z,v_l)$ is the light edge 
of $z$. 
We claim that $e \not\in s$, so that it is not present at the beginning of 
the path. 
If $e \in s$ then, since $s \in S(u)$, we have $u$, which queries $X$, 
in $\DNNF_{v_l}$. 
Together with the assumption that some node in $\DNNF_{v}$ queries $X$, 
we see that descendants of the two children $v_l,v$ of AND-node $z$ query 
the same variable which contradicts that $\DNNF$ is a DNNF. 
On the other hand, $e \in s''$. 
Now, the first node on the path where $e$ was introduced must have an edge of 
the form $(z,s_1),(v_l,s_1 \cup \{e\})$. 
But now we have a path from $(u,s)$ to $(z,s_1)$ with 
$X \in \DNNF_z \supset \DNNF_v$, contradicting the maximality of $v$.
\end{proof}

The next proposition says that on accepting paths
$P = \{(u_1,s_1),(u_2,s_2),\ldots (u_\ell,s_\ell)\}$ in the constructed 
OR-FBDD $\mathcal{F}$, the sequence of sets $(s_1,\ldots,s_\ell)$ behaves like
the sequence of states of a stack. 
We will use this characterization of paths in the proof of Lemma~\ref{samefn}.

\begin{sloppypar}
\begin{proposition}
Suppose that $\mathcal{F}$ has been constructed from a DNNF $\DNNF$ and that 
$P = \{(u_1,s_1),(u_2,s_2),\ldots (u_\ell,s_\ell)\}$ is a path in $\mathcal{F}$
consistent with a variable assignment $\theta$. 
If, for $j < i$, we have $e_1 \in s_j$, $e_1 \in s_i$, and
$e_2 \in s_i \setminus s_j$, then for no $k > i$ do we have both $e_2 \in s_k$
and $e_1 \not\in s_k$.
\label{stack}
\end{proposition}
\end{sloppypar}

\begin{proof}
Suppose the statement is false.
Then there must exist a Type 3 edge
$((w,s \cup \{e_1,e_2\}),(v_r,s \cup \{e_2\}))$ in the path $P$, where 
$w \in v_l$.
However, we cannot have $e_2 \in S(v_r)$: $e_2$ was an edge in $\DNNF_{v_l}$ 
because $(w,s \cup \{e_1,e_2\})$ was reachable in $\mathcal{F}$ meaning
that $e_2 \in S(w)$.
\end{proof}

\begin{lemma}
$\mathcal{F}$ computes the same function as $\DNNF$. 
That is, $\Phi_{\mathcal{F}}[\theta] = \Phi_{\DNNF}[\theta]$ for all variable
 assignments $\theta$.
\label{samefn}
\end{lemma}
\begin{proof}
Suppose that $\Phi_{\mathcal{F}}[\theta]=1$. Then there exists a path $P$ in 
$\mathcal{F}$ consistent with $\theta$ that ends in a 1-sink node. 

If $P=\{(u_1,s_1),(u_2,s_2),\ldots (u_\ell,s_\ell)\}$ has no Type 1 edges then 
it also has no Type 3 edges. Therefore $(u_1,\ldots,u_\ell)$ is a path of 
neutral edges in $\DNNF$ consistent with $\theta$ 
to a 1-sink with no AND-nodes along the way, thus $\Phi_{\DNNF}[\theta]=1$.

Otherwise, let
\[ S = \{(u_i,s_i),(u_{i+1},s_{i+1}),\ldots,(u_{i+j},s_{i+j})\}\]
be a sub-path of $P$. 
We say that $S$ corresponds to an accepting sub-DAG rooted at $u$ if there is a 
sub-DAG of $\DNNF$ rooted at the node $u$ whose OR nodes have fanout 1, AND 
nodes have full fanout, leaves are all 1-sinks under $\theta$, and whose edges 
are $u_i,u_{i+1},\ldots,u_{i+j}$.

Starting from an empty path in $\mathcal{F}$, we will work backwards from the 
end of $P$, adding two possible kinds of sub-path: with all Type 2 edges
removed, the first contains exactly one Type 1 edge followed by one or more 
Type 3 edges.
The second kind, with all Type 2 edges removed, contains only one Type 1 edge 
and no Type 3 edges.
It is possible to construct $P$ using these two types of subpath by 
Proposition~\ref{stack}, which says that
\[ s_1,s_2 \ldots s_\ell\]
is the sequence of states of a stack where, as we traverse the path $P$, its 
Type 1 edges push light edges while its Type 3 edges pop them. As $P$ is an
accepting path, we also have that $s_{\ell} = \emptyset$.
We will show that both of these types of additions give a path corresponding 
to an accepting sub-DAG rooted at all AND nodes mentioned in the Type 1 edges 
of the sub-path.

In the first case, say we add the sub-path $S_h$ to the tail path $S_t$ to 
form $S = S_h S_t$. 
Suppose $S_h$ contains the Type 1 edge $((u_h,s),(v_l,s \cup \{e\}))$, and 
that $S_t$ corresponds to an accepting sub-DAG respecting the first 
AND node in $S_t$, which we call $u_t$. 
We wish to show that $S$ corresponds to an accepting sub-DAG rooted at $u_h$. 
$S_h$ must contain a Type 3 edge popping $e$, hence there is a path in $\DNNF$ 
from $v_l$ to a 1-sink that is consistent with $\theta$. 
Therefore $S_h$ corresponds to an accepting sub-DAG rooted at $v_l$ (there are 
no AND-nodes along the way so the sub-DAG is the path). 
Further, since we can find a path of neutral edges in $\DNNF$ from the sibling 
node of $v_l$, $v_r$, to $u_t$ (these come from the portion of $P$ between the 
Type 3 edge popping $e$ and the Type 1 edge $((u_h,s),(v_l,s \cup \{e\}))$), 
$S$ corresponds to an 
accepting sub-DAG rooted at $v_r$. 
Therefore, $S$ corresponds to an accepting sub-DAG rooted at $u_h$.

In the second case, we add a Type 1 edge $((u_h,s),(v_l,s \cup \{e\}))$ to the 
tail path $S_t$, which corresponds to an accepting sub-DAG rooted at $u_t$, 
the first AND node in $S_t$. Then $u_t$ must appear in $\DNNF_{v_l}$. 
Otherwise the first Type 1 edge in $S_t$ comes after we pop $e$, but then 
$S_t$ began with the Type 3 edge popping $e$. This cannot happen because of 
our inductive assumption that we add sub-paths that begin with a Type 1 edge. 
So $S_t$ gives a path of neutral edges in $\DNNF$ from $v_l$ to $u_t$ (this is 
the sub-path in between the added Type 1 edge and the first Type 1 edge in 
$S_t$). 
Since $S_t$ corresponds to an accepting sub-DAG rooted at $u_t$, it also 
corresponds to an accepting sub-DAG rooted at $v_l$. 
Similarly, $S_t$ gives a neutral edge path in $\DNNF$ from $v_r$ to the first 
AND node mentioned after popping $e$. 
Again, from our inductive hypothesis, $S_t$ thus corresponds to an accepting 
sub-DAG rooted at $v_r$. 
Therefore $S$ corresponds to an accepting sub-DAG rooted at $u_h$.

Now suppose $\Phi_{\DNNF}[\theta]=1$. 
Then $\DNNF$ has an accepting sub-DAG $\DNNF_{\theta}$ respecting $\theta$. 
We can find an accepting path in $\mathcal{F}$ from edges coming from 
$\DNNF_{\theta}$. 
This path will follow a left-to-right traversal of $\DNNF_{\theta}$, 
keeping track of light edges pushed and popped. The Type 1 and Type 2 edge 
portions of this traversal (moving left down the tree) directly translate to 
the appropriate edges in $\mathcal{F}$. The necessary Type 3 edges for this
traversal also exist in $\mathcal{F}$ since $\DNNF_{\theta}$ only has 1-sinks. 
At the end of the traversal we will have popped all light edges and be at a 
1-sink in $\DNNF_{\theta}$ so we will be at a 1-sink for $\mathcal{F}$.
\end{proof}

\end{document}